\documentclass[journal,cspaper,compsoc]{IEEEtran} 

\usepackage[cmex10]{amsmath}
\usepackage{amssymb}
\usepackage{amsthm} 
\usepackage{amsmath}
\usepackage{graphicx}
\usepackage{subfigure}
\usepackage{color,soul} 
\usepackage{mathtools} 
\usepackage{forloop}  
\usepackage{algorithm,algpseudocode}
\usepackage{multicamfigs}
\usepackage{tabularx}
\usepackage{tkz-euclide}
\usetkzobj{all}
\usepackage{changepage} 
\usepackage{url}
\usepackage{color}
\usepackage{cite}

\usetikzlibrary{shapes.geometric,arrows}

\newcommand{\rotateRPY}[4][0/0/0]
{   \pgfmathsetmacro{\rollangle}{#2}
	\pgfmathsetmacro{\pitchangle}{#3}
	\pgfmathsetmacro{\yawangle}{#4}
	
	\pgfmathsetmacro{\newxx}{cos(\yawangle)*cos(\pitchangle)}
	\pgfmathsetmacro{\newxy}{sin(\yawangle)*cos(\pitchangle)}
	\pgfmathsetmacro{\newxz}{-sin(\pitchangle)}
	\path (\newxx,\newxy,\newxz);
	\pgfgetlastxy{\nxx}{\nxy};
	
	\pgfmathsetmacro{\newyx}{cos(\yawangle)*sin(\pitchangle)*sin(\rollangle)-sin(\yawangle)*cos(\rollangle)}
	\pgfmathsetmacro{\newyy}{sin(\yawangle)*sin(\pitchangle)*sin(\rollangle)+ cos(\yawangle)*cos(\rollangle)}
	\pgfmathsetmacro{\newyz}{cos(\pitchangle)*sin(\rollangle)}
	\path (\newyx,\newyy,\newyz);
	\pgfgetlastxy{\nyx}{\nyy};
	
	\pgfmathsetmacro{\newzx}{cos(\yawangle)*sin(\pitchangle)*cos(\rollangle)+ sin(\yawangle)*sin(\rollangle)}
	\pgfmathsetmacro{\newzy}{sin(\yawangle)*sin(\pitchangle)*cos(\rollangle)-cos(\yawangle)*sin(\rollangle)}
	\pgfmathsetmacro{\newzz}{cos(\pitchangle)*cos(\rollangle)}
	\path (\newzx,\newzy,\newzz);
	\pgfgetlastxy{\nzx}{\nzy};
	
	\foreach \x/\y/\z in {#1}
	{   \pgfmathsetmacro{\transformedx}{\x*\newxx+\y*\newyx+\z*\newzx}
		\pgfmathsetmacro{\transformedy}{\x*\newxy+\y*\newyy+\z*\newzy}
		\pgfmathsetmacro{\transformedz}{\x*\newxz+\y*\newyz+\z*\newzz}

	}
}

\tikzset{RPY/.style={x={(\nxx,\nxy)},y={(\nyx,\nyy)},z={(\nzx,\nzy)}}}

\tikzset{
	buffer/.style={
		draw,
		shape border rotate=-90,
		isosceles triangle,
		isosceles triangle apex angle=60,
		fill=red,
		node distance=2cm,
		minimum height=4em
	}
}
\newcommand\pgfmathsinandcos[3]{%
	\pgfmathsetmacro#1{sin(#3)}%
	\pgfmathsetmacro#2{cos(#3)}%
}

\usepackage{pgfplots}
\graphicspath{{images/}}

\DeclareMathOperator*{\argmin}{arg\,min}

\newcommand\proj{\mathcal{P}}
\newcommand{\f}{\mathit{f}} 

\newcommand{\x}{\mathbf{u}} 
\newcommand{\X}{\mathbf{U}} 
\newcommand{\bX}{\mathbf{X}} 
\newcommand{\bx}{\mathbf{x}} 
\newcommand{\ox}{\check{\x}} 
\newcommand{\hX}{\hat{\X}} 

\newcommand{\C}{\mathbf{C}} 
\newcommand{\cP}{\mathbf{P}} 
\newcommand{\R}{\mathbf{R}} 
\newcommand{\K}{\mathbf{K}} 
\newcommand{\Roi}{\mathcal{R}}
\newcommand{\lf}{\ell_\infty}

\theoremstyle{plain}\newtheorem{theorem}{Theorem}
\theoremstyle{plain}\newtheorem{proposition}{Proposition}
\theoremstyle{plain}
\theoremstyle{plain}
\theoremstyle{plain}\newtheorem{conjecture}{Conjecture}
\theoremstyle{remark} 
\theoremstyle{definition}\newtheorem{definition}{Definition} 

\newcommand\norm[1]{\left\lVert#1\right\rVert}

\begin{document}
\pgfplotsset{compat=1.12}

\newcommand{\reportname}{Bound and Conquer: Improving Triangulation by Enforcing Consistency}
\title{\reportname}
	\author{Adam~Scholefield,~\IEEEmembership{Member,~IEEE,}
	Alireza~Ghasemi,~\IEEEmembership{Student Member,~IEEE,}
		and~Martin~Vetterli,~\IEEEmembership{Fellow,~IEEE}
		%
		\thanks{This work was in part presented in \cite{ghasemi2015accuracy} and \cite{ghasemi2016shape}.}
		\thanks{Authors are with with the School of Computer and Communication Sciences, Ecole Polytechnique F\'ed\'erale de Lausanne (EPFL), CH-1015 Lausanne, Switzerland (e-mail: firstname.surname@epfl.ch).}
		\thanks{This work was supported by the Commission for Technology and Innovation (CTI) project no. 14842.1 PFES-ES and ERC Advanced Grant---Support for Frontier Research---SPARSAM Nr: 247006.}
		\thanks{A. Ghasemi was additionally supported by a Qualcomm Innovation Fellowship.}}

	\markboth{}%
	{}

\IEEEcompsoctitleabstractindextext
{%
	\begin{abstract}
	
	
		
	We study the accuracy of triangulation in multi-camera systems with respect to the number of cameras. We show that, under certain conditions, the optimal achievable reconstruction error decays quadratically as more cameras are added to the system. 
	Furthermore, we analyse the error decay-rate of major state-of-the-art algorithms with respect to the number of cameras. To this end, we introduce the notion of consistency for triangulation, and show that consistent reconstruction algorithms achieve the optimal quadratic decay, which is asymptotically faster than some other methods.
	Finally, we present simulations results supporting our findings. Our simulations have been implemented in MATLAB and the resulting code is available in the supplementary material.
	\end{abstract}
	%

}

\maketitle

\IEEEdisplaynotcompsoctitleabstractindextext

\IEEEpeerreviewmaketitle


\section{Introduction}

Cameras are finite; yet, we commonly assume Gaussian noise models with infinite tails. Clearly, this disparity, which appears across all of science, is offset by the approximation accuracy and mathematical convenience of the Gaussian distribution. However, by replacing cost functions based on the $\ell_2$-norm with their $\ell_\infty$-counterparts, the computer vision community has begun to investigate the feasibility of bounded noise models~\cite{hartley2007optimal,kahl2008multiple,donne2015point,freundlich2015exact}. 

In this paper, we continue this trend by formalising some of the benefits of bounded noise models in triangulation problems; i.e., problems which aim to estimate the three-dimensional (3-D) positions of feature points from their two-dimensional (2-D) projections. 

For example, consider triangulating from a set of calibrated cameras. In the noise-free, infinite-resolution case, the exact location of the world point can be reconstructed by intersecting rays originating from each camera. However, in practice, various sources of uncertainty mean that the rays do not necessarily intersect. Error minimisation techniques are thus employed with the goal of finding the most probable reconstructed world point~\cite{hartley1997triangulation,kahl2008multiple}. The degree to which we achieve this goal depends on two factors of the cost function: its correct modelling of the uncertainty and how accurately we can find its global minimum. 

The simplest approach is to construct an over-complete set of linear equations, each corresponding to one of the rays, and take the pseudo-inverse to find the least-squares solution. Although we find the global minimum, the cost function is not particularly meaningful and it is unlikely that it provides the best model of the underlying uncertainty. However, this technique performs reasonable well, particularly if the coordinates are correctly normalised~\cite{hartley1997triangulation,hartley1997defense}, and it is a good choice when time complexity is the principal concern.


Alternatively, we can attempt to minimise the $\ell_2$-norm of the reprojection error between the prospective 3-D points and the known image locations, which results in the maximum-likelihood estimator if we assume that the projected points are subjected to i.i.d. zero-mean Gaussian noise in the image plane. Although this assumption is very likely a much better approximation of the underlying uncertainty, the resulting cost function is non-convex and extremely difficult to solve. 

Although for a small number of views the global minimum can be found by polynomial root finding~\cite{hartley1997triangulation,kanatani2011optimal,kanatani2008triangulation,stewenius2005hard,hedborg2014robust,byrod2007fast,nordberg2008efficient}, the degree of the resulting polynomial grows quadratically with the number of views~\cite{stewenius2005hard} and thus this is, in general, not practical. Therefore, often, we have to resort to iterative approaches that only converge to a local minimum, or more time-consuming branch and bound techniques~\cite{agarwal2008practical}. On a positive note, it is possible to verify whether a solution is globally optimal in the $\ell_2$-sense~\cite{hartley2013verifying}. 


Inspired by these difficulties, the $\ell_\infty$-norm has recently been considered as a measure of the reprojection error~\cite{kahl2008multiple,hartley2007optimal,mittelmann2003independent}. Although this may not correspond to the best model of the underlying uncertainty, the resulting cost is a quasi-convex function and efficient algorithms can find the global optimum~\cite{donne2015point}. However, since the $\ell_\infty$-norm implicitly assumes a bounded noise model, care must be taken to limit the potential catastrophic effect of outliers~\cite{li2007practical,micusik2010localizing}.

In this paper, we analyse the performance of multi-camera systems and reconstruction algorithms under the assumption of bounded noise and pixelisation. We provide two main contributions. First, we prove that, under certain conditions, the highest achievable point localisation accuracy of a multi-camera system is quadratically related to the number of cameras in the system. Second, we introduce the notion of consistency and show that consistent reconstruction algorithms achieve the optimal quadratic decay rate. 

In the stereo case, there have been a number of excellent studies thoroughly analysing the point reconstruction error with relation to multiple parameters ~\cite{blostein1987error,rodriguez1990stochastic}. Furthermore, the error of depth estimation in linear camera arrays, has been analysed~\cite{raynor2013plenoptic}.

However, to the best of our knowledge, there is no work analysing arbitrary camera setups and certainly no work deriving fundamental scaling laws for the accuracy of point reconstruction, with respect to the number of cameras. Given the rapid increase in popularity of multi-camera systems, we hope this analysis provides a significant contribution.

Our work is inspired by results derived from frame quantisation. In particular, similar error decay rates have been derived for signal reconstruction from over-complete quantised projections~\cite{goyal1998quantized,cvetkovic1999source,beferull2003efficient}. In imaging terms, this corresponds to circular arrays of 2-D pixelised orthographic cameras. Furthermore, these results have recently been generalised to uniform bounded noise and any consistent estimate~\cite{rangan2001recursive,powell2016error}. We extend these results to more general arrays of 3-D cameras using central projections, which is more typical in the context of imaging.

In what follows, we give a very brief overview of the problem setup and formally define the triangulation problem. After introducing major state-of-the-art techniques, we present our main results on the error decay rate of consistent reconstruction algorithms and the best possible performance of multi-camera systems. 
Finally, simulations are provided to support our findings.

\section{Background}
\label{sec_imformation}

\subsection{Pinhole camera model}
As is typical, we assume the pinhole camera model, which we will now briefly summarise. For a more thorough introduction, we refer the reader to \cite{hartley2003multiple}.

As depicted in Figure \ref{pinhole_3d}, a pinhole camera projects a point in 3-D space, called the world point, to a point on the camera's 2-D image plane via a central projection.
In homogeneous coordinates, pinhole projection can be expressed as a linear matrix multiplication:
\begin{equation}
\ox_h=\cP\X_h,
\label{eq_pinhole_concise}
\end{equation}
where $\ox_h$ and $\X_h$ are the homogeneous representation of the projected point and world point, respectively. The $3\times 4$ matrix, $\cP$, is the camera matrix, which can be decomposed as
\begin{eqnarray}
\cP=\K\R[\mathbf{I}|-\C].
\end{eqnarray}
Here $\C$ denotes the camera centre, $\R$ the camera orientation, and $|$ the column-wise concatenation operator. Together, the camera centre and the camera orientation form the extrinsic, or pose parameters. Finally, the matrix $\K$ contains the intrinsic camera parameters: namely, the focal length $\f$ and the coordinates of the principal point $p=(c_x,c_y)$.

In this paper, it will generally be more convenient to work with the Cartesian coordinates of Euclidean geometry. In this case, we will write the non-linear pinhole projection as
\begin{equation}
\ox=\proj(\X)=\frac{1}{\ox_h[3]}\begin{bmatrix}\ox_h[1]\\ \ox_h[2] \end{bmatrix},
\label{eq_pinhole_concise_metric}
\end{equation}
where $\ox$ and $\X$ are the Cartesian coordinate representations of the projected point and world point, respectively, and $\ox_h$ is the homogeneous representation of the projected point.

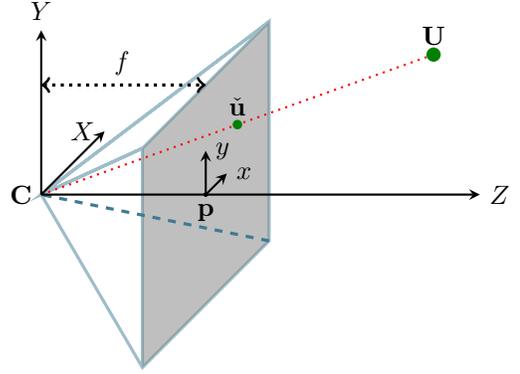
\begin{figure}
	\centering
	\begin{tikzpicture}[scale=.9] 
	\rotateRPY{0}{90}{0}
	\begin{scope}[RPY]
		\rotateRPY{0}{0}{0}
		\begin{scope}[RPY]
			\pgfmathsetmacro\AngleFuite{135}
			\pgfmathsetmacro\coeffReduc{.8}
			\pgfmathsetmacro\clen{2}
			\pgfmathsetmacro\foc{3}
			\pgfmathsinandcos\sint\cost{\AngleFuite}

			\path coordinate (O) at (0,0,0)
			coordinate (C) at (-3,-2,\foc) 
			coordinate (B) at (-3,2,\foc) 
			coordinate (A) at (3,2,\foc) 
			coordinate (S) at (3,-2,\foc)
			coordinate (O) at (0,0,0)
			coordinate (U) at (3,1.4,6)
			coordinate (u) at (1.5,0.7,\foc)
			coordinate (p) at (0,0,\foc);

			\draw[color=black!50!cyan,very thick,opacity=.5] (C) -- (B) -- (A) -- (S) -- (C) -- (O) -- (A) -- (O) -- (B) --  (O) ; 
			\fill[color=gray,fill opacity=.5] (C) -- (B) -- (A) -- (S) -- (C);
			\draw[color=black!50!cyan,very thick,dashed] (S) -- (O);

			\draw[color=red,thick,dotted] (U) -- (O);
			
			
			\fill[color=black!50!green] (U) circle (3pt);
			\fill[color=black!50!green] (u) circle (2pt);
			\fill[color=black] (p) circle (1pt);
			\draw[very thick,dotted,<->] (0,2,0) -- (0,2,\foc) node [midway, above] {$\f$};
			
			\draw[thick,-stealth,black]  (O)  -- (0,0,8) node[right]{$Z$} ; 
			\draw[thick,-stealth,black]   (O)  -- (3,0,0) node[left]{$X$};
			\draw[thick,-stealth,black] (O)  -- (0,3,0) node[above]{$Y$};
			
			\draw[thick,-stealth,black]   (p)  -- (1,0,\foc) node[right]{$x$};
			\draw[thick,-stealth,black] (p)  -- (0,0.8,\foc) node[right]{$y$};
			
			\node[left] at (O) {$\C$};
			\node[below] at (p) {$\mathbf{p}$};
			\node[above] at (U) {$\X$};
			\node[above] at (u) {$\ox$};
		\end{scope}	
	\end{scope}	
	\end{tikzpicture}  
	\caption{An example of central projection in a pinhole camera}
	\label{pinhole_3d}
\end{figure}

\subsection{Sources of uncertainty}
\label{sec_pertmodel}

Due to various sources of uncertainty, the true image location $\ox$ is perturbed to yield the measurement $\x$. The error term, incorporates both deterministic (e.g. pixelisation) and random (e.g. measurement noise) perturbation.

\subsubsection{Pixelisation}
Even when we are in a hypothetical noiseless scenario, we still have to deal with the uncertainty caused by the finite resolution of the camera sensors. 

This source of uncertainty, which we call pixelisation, is deterministic and, when projected back to the world space, leads to semi-infinite regions in the world space instead of rays. These regions, originating from the boundaries of the pixels, partition the world space into a finite number of regions. Each region consists of all world points whose projections map to the same pixel. 
Figure \ref{fig_pixelisation} depicts a simple example of such a partitioning for a camera with sixteen pixels.

When multiple cameras view the same region of interest, these regions intersect producing a finite number of regions, each corresponding to a particular combination of pixels in the cameras. Clearly, smaller regions lead to a smaller uncertainty.

\begin{figure}
\centering
\begin{tikzpicture}[scale = .8]
	\dispcamfullview{1.5}{1}{4}{10}{0}{0}{0};	

\end{tikzpicture}
\caption{Pixelisation in a digital camera.}
\label{fig_pixelisation}	
\end{figure}
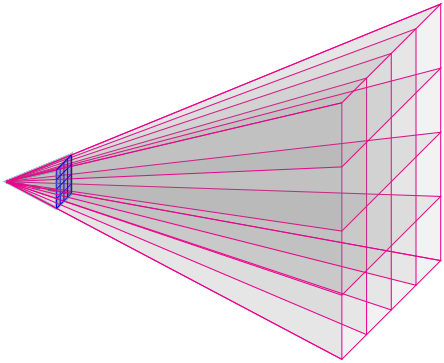

\subsubsection{Non-deterministic sources of uncertainty}
In reality, pixelisation is not the only source of uncertainty, with additionally noises arising from image sensor noise, as well as the error of corner localisation algorithms. 

These perturbation sources, combined with pixelisation, eventually lead to ambiguity in measuring the exact location of an image point. This can be modelled as an additive noise, yielding
\begin{equation}
\x=\ox+\boldsymbol{\epsilon}=\proj(\X)+\boldsymbol{\epsilon}.
\end{equation}


\subsubsection{Bounded noise models}
In this paper, we will be interested in bounded noise models:
\begin{equation}
    \|\boldsymbol{\epsilon}\|_q\leq \delta,
\end{equation}
where $q$ specifies the shape of the bounded noise and $\delta$ is referred to as the bandwidth of the noise.

The main advantage of bounded noise is that it allows us to completely dismiss regions of the solution space as impossible and, as we will show, the size of the remaining feasible region decays in such a way that the squared reconstruction error decays quadratically with the number of measurements.

On may question the applicability of bounded noise. Clearly, in the presence of outliers, additional work must be done to prevent these methods breaking down. However, if outlier techniques are applied, bounded approaches can be useful. It would be interesting to fully investigate this comparison, for practical problems; however, this is beyond the scope of this theoretical study.

\subsection{The triangulation problem}
This paper focuses on triangulation---a fundamental problem in multiple-view geometry, which, as well as being interesting in its own right, provides a basic building block for many higher-level computer vision tasks, such as visual metrology, Simulataneous Localisation And Mapping (SLAM), and Structure from Motion (SfM). 

The aim is to recover the location of an unknown 3-D point $\X$ from its projections in $M$ calibrated cameras; i.e. the camera matrices $\cP_1$ through $\cP_M$ are known and we estimate $\X$ from the measured projections $\x_1,\cdots,\x_M$. To be able to prove things about triangulation, we formally define it as follows:

\begin{definition}\label{def_triang}
	A \textit{triangulation problem} takes as input
	\begin{eqnarray}
		T=\{(\x_i,\cP_i)\,|\,1\leq i \leq M\},
	\end{eqnarray}
	and estimates the underlying unknown 3-D world point $\X$ as follows:
	\begin{equation}
        \hX=\argmin_\bX \sum_{i=1}^M \left\|\x_i-\proj_i(\bX)\right\|^p_{p'}.
        \label{eq_triang}
    \end{equation}
    Here, $\proj{i}{.}$ denotes the projection operator corresponding to the camera matrix $\cP_i$ and the $(\x_i,\cP_i)$ pairs denote the camera matrices of the $M$ cameras along with the projections of the unknown 3-D world point $\X$ on their image planes.
\end{definition}
In the previous definition, $p'$ and $p$, are known as the image-space and residual-space norms, respectively. When we talk about algorithms that minimises the $(\ell_{p'},\ell_p)$-norm of the reprojection error, for some particular $p'$ and $p$, we are referring to the image-space and residual-space norms in this order.

When $p=\infty$, we assume that \eqref{eq_triang} becomes
\begin{equation}
    \hX=\argmin_\bX \max_{i=1..M} \left\|\x_i-\proj_i(\bX)\right\|_{p'}.\label{eq_triang_linf}
\end{equation}
  


\subsection{Equivalence with camera localisation}

As an aside, we briefly mention the connection between triangulation and a restricted version of camera localisation. In particular, by replacing the $M$ cameras and single feature point of the triangulation problem with $M$ feature points and a single camera, one can easily show that the triangulation problem, as just defined, is mathematically equivalent to localising a single camera from the image points of $M$ feature points at known locations. The catch is that the equivalence is only valid if one assumes that the orientation of the camera is known.

Of course, in most practical problems, the camera orientation is not known and heuristics must be applied if one wishes to adapt triangulation techniques to camera localisation. Since in this paper focuses on a mathematical analysis, we restrict our analysis to triangulation but note that the scaling laws we derive also apply to this restricted version of camera localisation. Deriving the scaling laws for the full camera localisation problem is an interesting open research problem.
\section{Reconstruction algorithms}
\label{sec_sota}
In this section, we briefly review the main techniques used to solve triangulation and other geometric reconstruction algorithms.

\subsection{Linear triangulation}
\label{sec_lintr}
The simplest approach to triangulation is to construct a linear system of equations. Let $\boldsymbol{p_l}^T$ be the $l$-th row of a a camera matrix $\cP$. Then,
\begin{equation}
\ox[1]=\frac{\ox_h[1]}{\ox_h[3]}=\frac{\boldsymbol{p_1}^T\X_h}{\boldsymbol{p_3}^T\X_h},
\end{equation}
and similarly for $\ox[2]$.
Therefore, for a single camera, we have
\begin{equation}
\begin{bmatrix}
\ox[1]\boldsymbol{p_3}^T- \boldsymbol{p_1}^T\\
\ox[2]\boldsymbol{p_3}^T- \boldsymbol{p_2}^T\\
\end{bmatrix}\X_h=\boldsymbol{0}.
\end{equation}
For $M$ cameras, we can stack the $2M$ equations into a matrix $\boldsymbol{A}\in\mathbb{R}^{2M\times 4}$, with $\boldsymbol{A}\X_h=\boldsymbol{0}$.

This equation can be solved efficiently with standard techniques, such as the Singular Value Decomposition (SVD). 
However, due to the conversion to homogeneous coordinates, this does not minimise the desired cost; i.e., a norm of the residual vector.

The advantages of linear triangulation are its speed and simplicity and it performs particularly well when the cameras are at almost the same depth to the point of interest. Its robustness can also be further improved by normalising the focal length and other metric distances in the problem instance~\cite{hartley1997triangulation,hartley1997defense}.

\subsection{Reprojection error minimisation}
If increased computational resources are available, one can directly attempt to minimise~\eqref{eq_triang} for different values of $p'$ and $p$.

The most common choice is the $(\ell_2,\ell_2)$-norm, but, as explained in the introduction, the resulting cost function is non-convex and often difficult to solve exactly. Therefore, often, linear triangulation is used to initialise a gradient descent approach thus accepting convergence to a local minimum~\cite{hartley2003multiple,hartley2013verifying}.

Alternatively, more computationally intensive branch and bound techniques can be used, which guarantee convergence to the global minimum. In particular, in~\cite{kahl2008practical}, the authors present a branch and bound technique for minimising the  $(\ell_2,\ell_2)$, $(\ell_2,\ell_1)$ and $(\ell_1,\ell_1)$-norms.

The $\lf$-norm leads corresponds to the assumption of bounded noise. In the case of the $(\ell_2,\ell_\infty)$-norm, the image points are bounded to circles on the image plane, which back project as cones in the $3$-D world space. Consequently, Second Order Cone Programming (SOCP) can be applied~\cite{kahl2008multiple}.


As with all approaches based on bounded noise, $\lf$-based methods suffer from being extremely sensitive to even a single outlier. Therefore, it is critical that these techniques are either combined with standard outlier removal techniques or, as has been recently proposed, relaxations applied~\cite{micusik2010localizing,li2007practical,seo2009outlier}. 

\section{Consistent reconstruction and the accuracy of multi-camera systems}
\label{sec_proofs}
In this section, we present two results concerning the accuracy of multi-camera systems as more cameras are added to the system. 

In the first, we prove a lower bound for the average reconstruction error of any multi-camera system, over a region of interest. This lower bound decreases quadratically as more cameras are added to the system.

Next, we introduce the concepts of consistency and consistent reconstruction and then prove that the reconstruction error of a consistent reconstruction algorithm is upper bounded by a term that decreases quadratically as more cameras are added to the system. Therefore, consistent reconstruction algorithms achieve the optimal error decay rate. 

This is not necessarily the case for other algorithms: linear triangulation, for example, can be shown to yield a linear error decay rate with respect to the number of cameras.

\subsection{Lower bound for the accuracy of a multi-camera system}

We would like to lower-bound the best possible reconstruction error achievable by a multi-camera system, using any possible triangulation algorithm. To do this, we first formally define what we mean by a triangulation algorithm.
\begin{definition}
	A \emph{triangulation algorithm} is any mapping from $T=\{ (\x_i,\cP_i) : 1\leq i\leq M \}$ to $\hX\in\mathbb{R}^3$.
	\label{def__tri_alg}
\end{definition}

Since we are seeking a lower-bound, it makes sense to limit the uncertainties in the system. Therefore, in the following theorem, we assume that pixelisation is the only source or uncertainty; however, note that the size of pixels is arbitrary and thus this is a mild assumption that is interesting even if an image point can be localised with subpixel precision.

\begin{theorem}\label{th_lowerbound}
    Consider a multi-camera system of $M$ cameras, each with an $N\times N$ pixel image sensor and define a fixed region of interest, $\Roi$, with a finite non-zero volume.
	
	If we assume that the only source of uncertainty is pixelisation, the expected reconstruction error of any triangulation algorithm is lower-bounded by a term that is inverse-quadratically dependent on the number of cameras; i.e.,
	\begin{equation}
	\mathbb{E}\left(\norm{\hX-\X}^2\right)=\Omega\left(\frac{1}{M^2}\right),
	\end{equation}
	where $\X\in\Roi$ is any point in the region of interest, and $\hX$ is the result of reconstructing $\X$, from its images in the multi-camera system, using any triangulation algorithm. Here, the expectation is taken over the location of the point $\X$ in the region of interest.
\end{theorem}

\begin{proof}
For the proof, please refer to the supplementary material.\nocite{goodman1986upper}
\end{proof}

The above theorem states that, under certain assumptions, no triangulation algorithm can do better than a quadratic decay with respect to the number of cameras, regardless of the camera setup. 
In what follows, we show that, if the camera array is properly constructed, the expected reconstruction error of certain triangulation algorithms can be upper-bound by a term that decays quadratically with the number of cameras. Therefore, in doing so, we show that these triangulation algorithms reach the best possible decay rate. So which triangulation algorithms achieve this optimal decay? It turns out that the key property is \emph{consistency}.

\subsection{Consistency and Consistent Reconstruction}
A key advantage of bounded noise models is that, given a noisy image point, we can restrict the true image point to a finite $2$-D region, which we call an \emph{image-space consistency region}:
\begin{eqnarray}
    I_{\x,\delta}=\{  \bx\in\mathbb{R}^2: \|  \bx-\x  \|_q \leq \delta \}.
\end{eqnarray}
The shape of these regions depends on the type of bounded noise: a circle when $q=2$, a diamond when $q=1$ and a square when $q=\infty$. Note that, we assume that $q\geq 1$ so the norm is properly defined. In this case, the image-space consistency region is convex.

If we back-project an image-space consistency region into the world space, we obtain a convex $3$-D region, which we call a \emph{world-space consistency region}:
\begin{eqnarray}
W_{\x,\delta,\cP}=\{  \bX\in\mathbb{R}^3: \|  \x - \proj(\bX) \|_{q} \leq \delta \}.
\end{eqnarray}  
Again, depending on the type of bounded noise ($q$), these $3$-D regions can be either a cone, a diamond-based pyramid, or a square-based pyramid.

We know that the true $3$-D point must lie in the intersection of all world-space consistency regions:
\begin{equation}
V_{\x,\delta,\mathfrak{P}}=\{ \bX \in \mathbb{R}^3:  \wedge_{i=1}^M  \| \x_i - \proj_i(\bX) \|_{q} \leq \delta \},
\label{eq_common_cons_region}
\end{equation}
where $\mathfrak{P}=\{\cP_i: i\in[1,M]\}$.
For a particular geometric reconstruction problem, we call this region the \emph{consistent region} and any estimate that lies within it a \emph{consistent estimate}. In addition, if a reconstruction algorithm always returns a \emph{consistent estimate}, we call it a \emph{consistent reconstruction algorithm}. This is stated more formally in the following definition.

\begin{definition}
A triangulation algorithm is \emph{consistent}, over the region of interest $\Roi$, if
\begin{equation}
	\x_i\in I_{\proj_i(\X),\delta}\quad \forall i\in[1, M]\qquad\Rightarrow\qquad\hX\in V_{\x,\delta,\mathfrak{P}},
\end{equation}
for any $\X\in\Roi$ and valid projection matrices $\mathfrak{P}=\{\cP_i:1\leq i\leq M\}$. Note that a consistent triangulation algorithm returns no estimate when $V_{\x,\delta,\mathfrak{P}}=\varnothing$.
	\label{def__cons_alg}
\end{definition}

\subsection{Finding a consistent estimate}

As stated in the following proposition, $\ell_\infty$-based triangulation is \emph{consistent}.

\begin{proposition}
    Consider a multi-camera system viewing a point and assume that the image points are subjected to $\ell_q$-norm bounded noise:
    \begin{equation*}
        \| \x_i - \proj_i(\bX) \|_{q} \leq \delta\quad\text{for }i=1...M.
    \end{equation*}
    Then, any algorithm that minimises the $(\ell_q,\ell_\infty)$-norm of the reprojection error is a \textit{consistent} triangulation algorithm.
\end{proposition}
\begin{proof}
	For the proof, please refer to the supplementary material.
\end{proof}

For $\ell_2$-bounded noise, the consistent regions are cones and the SOCP technique outlined by Kahl et.al.~\cite{kahl2008multiple} returns a consistent estimate.

For $\ell_\infty$-bounded noise, the following simple linear program (LP) can be used to find a consistent estimate. Recall that, for the linear triangulation algorithm, we used $\boldsymbol{p_l}^T$ to denote the $l$-th row of a camera matrix $\cP$. To avoid homogeneous coordinates, let's separate the first three elements from the last: $\boldsymbol{p_l}^T = \begin{bmatrix} \boldsymbol{\bar{p}_l}^T & p_{l4} \end{bmatrix}$. Then,
\begin{equation}
\ox[1]=\frac{\ox_h[1]}{\ox_h[3]}=\frac{\boldsymbol{\bar{p}_1}^T\X+p_{14}}{\boldsymbol{\bar{p}_3}^T\X+p_{34}},
\end{equation}
and similarly for $\ox[2]$. For bounded noise, with bandwidth $\delta$,
\begin{equation*}
    \x[i]-\delta\leq\ox[i]\leq\x[i]+\delta,\qquad i=1,2.
\end{equation*}
Therefore, for a single camera, we have
\begin{equation*}
\begin{bmatrix}
(\x[1]-\delta)\boldsymbol{\bar{p}_3}^T- \boldsymbol{\bar{p}_1}^T\\
(\x[2]-\delta)\boldsymbol{\bar{p}_3}^T- \boldsymbol{\bar{p}_2}^T\\
\boldsymbol{\bar{p}_1}^T - (\x[1]+\delta)\boldsymbol{\bar{p}_3}^T\\
\boldsymbol{\bar{p}_2}^T - (\x[2]+\delta)\boldsymbol{\bar{p}_3}^T
\end{bmatrix}\X\leq\begin{bmatrix}p_{14}-(\x[1]-\delta)p_{34}\\
p_{24}-(\x[2]-\delta)p_{34}\\
(\x[1]+\delta)p_{34}-p_{14}\\
(\x[2]+\delta)p_{34}-p_{24}
\end{bmatrix}.
\end{equation*}
Note that, here, we have assumed that the point is in front of the camera so that $\boldsymbol{\bar{p}_3}^T\X+p_{34}>0$.

For $M$ cameras, we can stack the $4M$ inequalities producing a matrix $\boldsymbol{A}\in\mathbb{R}^{4M\times3}$ and vector $\boldsymbol{b}\in\mathbb{R}^{4M}$, such that $\boldsymbol{A}\X\leq\boldsymbol{b}$. Any point satisfying these constraints is consistent and thus the following LP will return a consistent estimate:
\begin{equation}
    \hX = \arg\min_\X \boldsymbol{c}^T\X,\quad\text{s.t. }\boldsymbol{A}\X\leq\boldsymbol{b}.\label{eq:LP}
\end{equation}
Here, the vector $\boldsymbol{c}\in\mathbb{R}^3$ dictates which point in the consistent region is the optimum. As we show in the next section, for many cameras, any consistent estimate is good and thus in our simulations we use a standard LP solver with $\boldsymbol{c}=\boldsymbol{0}$.

As an aside, we note that more complex LPs can be used to select a more desirable point in the consistent region. As just stated, this will only be beneficial for small $M$, since asymptotically they will achieve the same performance. 

For example, one can take the previous linear program and remove all redundant constraints. Assume this produces $\bar{M}$ equivalent non-redundant constraints: $\boldsymbol{\bar{A}}\X\leq\boldsymbol{\bar{b}}$. Then, a LP that finds the point in the consistent region that minimises the average distance to the $\bar{M}$ planes at the limit of the constraints can be designed as follows. Let $\boldsymbol{\bar{a}_l}^T$ be the $l$-th row of $\boldsymbol{\bar{A}}$. Then, the minimum distance from any point $\boldsymbol{V}\in\mathbb{R}^3$ to the $l$-th plane, is
\begin{equation}
    d_l:=\frac{\boldsymbol{\bar{a}_l}^T\boldsymbol{V}-\bar{b}_l}{\|\boldsymbol{\bar{a}_l}\|_2}.
\end{equation}
We thus wish to solve
\begin{equation}
    \arg\min_\X \sum_l |d_l|,\quad\text{s.t. }\boldsymbol{\bar{A}}\X\leq\boldsymbol{\bar{b}}.
\end{equation}
It is a standard exercise in linear programming to convert this problem into standard form using $\bar{M}$ auxiliary variable, one for each distance. 


\subsection{Error Decay in Consistent Reconstruction}

We will shortly present a theorem stating that, for circular camera arrays, the expected reconstruction error of consistent algorithms is upper bounded by a term that decays quadratically with the number of cameras. However, since simulations suggest that the result holds in many more cases, including linear camera arrays and even quite general random setups, we present the following more general conjecture, which is numerically tested in the following section.

\begin{conjecture}
    Define the region of interest, $\Roi$, to be a sphere of finite radius $r$ and place a point anywhere in this region.
    Place $M$ cameras inside a larger finite radius sphere, with the same centre as $\Roi$, i.i.d. uniformly at random such that they all see the whole region of interest.
	
	Furthermore, assume that the images of the world point in the cameras are perturbed with uniform bounded noise; i.e., for the world point $\X$, the image $\x_i$ in the $i$-th camera is computed as
	\begin{eqnarray}
	\x_i=\proj_i(\X)+\boldsymbol{\epsilon}_i,
	\end{eqnarray}
	where $\boldsymbol{\epsilon}_i$ is zero-mean uniform bounded random satisfying $\|\boldsymbol{\epsilon}_i\|_q\leq\delta$.
	
	In this situation, the expected reconstruction error of any consistent triangulation algorithm is upper-bounded by a term which decreases quadratically with the number of cameras; i.e.,
	\begin{equation}
	\mathbb{E}\left(\norm{\hX-\X}^2\right)=\mathcal{O}\left(\frac{1}{M^2}\right),
	\label{eq_ere_upbound_conj}
	\end{equation}
	where $\X\in\Roi$ is any point in the region of interest, and $\hX$ is the result of reconstructing $\X$, from its images in the multi-camera system, using a consistent triangulation algorithm. Here, the expectation is taken over both the noise vector $\mathbf{\epsilon}$ and the camera locations.
	\label{conj_upbound}
\end{conjecture}

Now, the theorem for circular camera arrays. Once again, simulation results are presented in the following section to support this theorem.

\begin{theorem}
	Place $M$ cameras in a plane, i.i.d. uniformly at random on a finite radius circle oriented towards the centre of the circle. Define the region of interest, $\Roi$, to be the intersection of the field of view of all cameras as $M\rightarrow\infty$ and place a point anywhere in this region.
	
	Furthermore, assume that the images of the world point in the cameras are perturbed with uniform bounded noise; i.e., for the world point $\X$, the image $\x_i$ in the $i$-th camera is computed as
	\begin{eqnarray}
	\x_i=\proj_i(\X)+\boldsymbol{\epsilon}_i,
	\end{eqnarray}
	where $\boldsymbol{\epsilon}_i = [\mathbf{\epsilon}_{i,x},\mathbf{\epsilon}_{i,y}]^T$ and $\mathbf{\epsilon}_{i,x}$, $\mathbf{\epsilon}_{i,y}$ are zero-mean uniform bounded random variables with bandwidth $\delta$.
	
	In this situation, the expected reconstruction error of any consistent triangulation algorithm is upper-bounded by a term which decreases quadratically with the number of cameras; i.e.,
	\begin{equation}
	\mathbb{E}\left(\norm{\hX-\X}^2\right)=\mathcal{O}\left(\frac{1}{M^2}\right),
	\label{eq_ere_upbound_thm}
	\end{equation}
	where $\X\in\Roi$ is any point in the region of interest, and $\hX$ is the result of reconstructing $\X$, from its images in the multi-camera system, using a consistent triangulation algorithm. Here, the expectation is taken over both the noise and the camera locations.
	\label{th_upbound_3d}
\end{theorem}


\begin{proof}
	The proof makes use of \cite[Corollary 6.2]{powell2016error} and appears in the supplementary material.
\end{proof}
%
%
\begin{figure*}[t!]
	\begin{minipage}{0.5\columnwidth}
\begin{tikzpicture}[scale=1]
	\begin{axis}[
	    name=axis_name,
		width=1.9\columnwidth,
        height=1.9\columnwidth,
		xlabel=\large{$\log_2 M$},
		ylabel={\large{$\log_2 \mathcal{E}$}},
		legend style={font=\small},xmin=1,ymax=4,
		legend cell align=left,
		]

		\addplot coordinates{
(1.00000000, 18.68840069)
(1.58496250, -1.36626111)
(2.00000000, -1.79548867)
(2.32192809, -2.27428044)
(2.58496250, -2.59161981)
(3.16992500, -3.32442417)
(3.58496250, -3.76413192)
(3.90689060, -4.01742838)
(4.39231742, -4.46934074)
(4.80735492, -4.91214878)
(5.20945337, -5.24581156)
(5.64385619, -5.52634646)
(6.04439412, -5.82354128)
(6.47573343, -6.06979978)
(6.89481776, -6.24842688)
(7.31288296, -6.43894866)
(7.73470962, -6.63091234)
(8.15987134, -6.74483703)
(8.57742883, -6.81719402)
(9.00000000, -6.90165515)
		};
		\addlegendentry{\small Lin. triang.};
		
		\addplot coordinates{
(1.00000000, 4.09153382)
(1.58496250, -1.96836554)
(2.00000000, -2.71876520)
(2.32192809, -3.19790059)
(2.58496250, -3.69956672)
(3.16992500, -4.53267456)
(3.58496250, -5.12061633)
(3.90689060, -5.47086837)
(4.39231742, -6.08904465)
(4.80735492, -6.53811021)
(5.20945337, -7.04383135)
(5.64385619, -7.53085544)
(6.04439412, -7.95137990)
(6.47573343, -8.47795940)
(6.89481776, -8.86797626)
(7.31288296, -9.29185160)
(7.73470962, -9.69783722)
(8.15987134, -10.22404040)
(8.57742883, -10.60459108)
(9.00000000, -11.04296233)
		};
		\addlegendentry{\small Min. ($\ell_2$,$\ell_2$)-norm};
		
		\addplot coordinates{
(1.00000000, 38.45519418)
(1.58496250, -1.77804063)
(2.00000000, -2.61872805)
(2.32192809, -3.02216284)
(2.58496250, -3.50418216)
(3.16992500, -4.44119722)
(3.58496250, -4.96955947)
(3.90689060, -5.29752310)
(4.39231742, -5.84452712)
(4.80735492, -6.24215653)
(5.20945337, -6.51279216)
(5.64385619, -6.90865063)
(6.04439412, -7.29620456)
(6.47573343, -7.66638085)
(6.89481776, -7.99293982)
(7.31288296, -8.44109664)
(7.73470962, -8.69826191)
(8.15987134, -9.14895083)
(8.57742883, -9.46327288)
(9.00000000, -9.81863599)
		};
		\addlegendentry{\small Min. ($\ell_2$,$\ell_\infty$)-norm};
	
		\addplot coordinates{
(1.00000000, 1.16074694)
(1.58496250, -0.19664551)
(2.00000000, -0.96097295)
(2.32192809, -1.50621164)
(2.58496250, -2.09405666)
(3.16992500, -3.07645682)
(3.58496250, -3.96755621)
(3.90689060, -4.53154129)
(4.39231742, -5.47422759)
(4.80735492, -6.27307796)
(5.20945337, -6.98519845)
(5.64385619, -7.97426885)
(6.04439412, -8.77145756)
(6.47573343, -9.65322808)
(6.89481776, -10.48470124)
(7.31288296, -11.22827840)
(7.73470962, -12.18308976)
(8.15987134, -12.97935168)
(8.57742883, -13.84411248)
(9.00000000, -14.56075705)
		};
		\addlegendentry{\small LP for ($\ell_\infty$,$\ell_\infty$)-norm};
		
	\end{axis}
	\node [below=1.2cm] at (axis_name.south) {(a) Additive $\ell_\infty$-norm bounded noise.};
\end{tikzpicture}
	\end{minipage}\hspace{5.2cm}
	\begin{minipage}{0.5\columnwidth}
\begin{tikzpicture}[scale=1]
	\begin{axis}[
		width=1.9\columnwidth,
        height=1.9\columnwidth,
		xlabel=\large{$\log_2 M$},
		ylabel={\large{$\log_2 \mathcal{E}$}},
		legend style={font=\small},xmin=1,ymax=4,
		legend cell align=left,
		]
		\addplot coordinates{
(1.00000000, 6.15331253)
(1.58496250, -1.60482111)
(2.00000000, -2.19561062)
(2.32192809, -2.52182219)
(2.58496250, -3.09040889)
(3.16992500, -3.77218727)
(3.58496250, -4.13739702)
(3.90689060, -4.49617021)
(4.39231742, -4.89803693)
(4.80735492, -5.36931574)
(5.20945337, -5.81172242)
(5.64385619, -6.05403950)
(6.04439412, -6.34862248)
(6.47573343, -6.62304049)
(6.89481776, -6.87553276)
(7.31288296, -7.07534891)
(7.73470962, -7.23073535)
(8.15987134, -7.38942005)
(8.57742883, -7.50229708)
(9.00000000, -7.59961980)
		};
		\addlegendentry{\small Lin. triang.};
		\addplot coordinates{
(1.00000000, 1.05937024)
(1.58496250, -2.32228901)
(2.00000000, -2.99811522)
(2.32192809, -3.49493333)
(2.58496250, -4.04490274)
(3.16992500, -4.84087507)
(3.58496250, -5.45857252)
(3.90689060, -5.81243506)
(4.39231742, -6.49971264)
(4.80735492, -6.93203794)
(5.20945337, -7.41049322)
(5.64385619, -7.91402921)
(6.04439412, -8.41602739)
(6.47573343, -8.77329178)
(6.89481776, -9.25036191)
(7.31288296, -9.70083353)
(7.73470962, -10.14262796)
(8.15987134, -10.65462187)
(8.57742883, -11.05389575)
(9.00000000, -11.47731109)
		};
		\addlegendentry{\small Min. ($\ell_2$,$\ell_2$)-norm};
		\addplot coordinates{
(1.00000000, 0.43731299)
(1.58496250, -2.25339102)
(2.00000000, -2.80721716)
(2.32192809, -3.52660053)
(2.58496250, -4.03003906)
(3.16992500, -5.12795894)
(3.58496250, -5.89856739)
(3.90689060, -6.30656124)
(4.39231742, -7.32247517)
(4.80735492, -8.13004040)
(5.20945337, -8.81319291)
(5.64385619, -9.62088308)
(6.04439412, -10.41848025)
(6.47573343, -11.24963430)
(6.89481776, -11.99554512)
(7.31288296, -12.84131190)
(7.73470962, -13.67146245)
(8.15987134, -14.55176130)
(8.57742883, -15.46630266)
(9.00000000, -16.27764093)
		};
		\addlegendentry{\small Min. ($\ell_2$,$\ell_\infty$)-norm};
		\addplot coordinates{
(1.00000000, 1.01974836)
(1.58496250, -0.20238971)
(2.00000000, -0.78751561)
(2.32192809, -1.30699612)
(2.58496250, -1.84172342)
(3.16992500, -2.69027425)
(3.58496250, -3.29635208)
(3.90689060, -3.71305949)
(4.39231742, -4.46771166)
(4.80735492, -4.97635653)
(5.20945337, -5.64578176)
(5.64385619, -6.19939012)
(6.04439412, -6.80146849)
(6.47573343, -7.37916130)
(6.89481776, -7.94021128)
(7.31288296, -8.50662607)
(7.73470962, -9.13661637)
(8.15987134, -9.68525326)
(8.57742883, -10.25681790)
(9.00000000, -10.79159839)
		};
		\addlegendentry{\small LP for ($\ell_\infty$,$\ell_\infty$)-norm};
	\end{axis}
	\node [below=1.2cm] at (axis_name.south) {(b) Additive $\ell_2$-norm bounded noise.};
\end{tikzpicture}
    \end{minipage}
	\caption{Verification of Conjecture~\ref{conj_upbound}. Expected squared error ($\mathcal{E}$) as a function of the number of cameras ($M$).
	}
	\label{fig_verification_conj1}
\end{figure*}
\section{Simulations}
\label{sec_simultations}
We now present simulations to verify our theoretical results.
To approximate the expected reconstruction error, we take the average of many realizations. This means that any algorithm we use is run many times and thus has to be extremely robust. Unfortunately, this prevented us from using the branch and bound technique proposed in~\cite{kahl2008practical}, since we were unable to prevent their implementation from crashing for certain realizations.
Therefore, for the ($\ell_2$,$\ell_2$)-norm we used the non-linear approach from Hartley and Zisserman~\cite{hartley2003multiple}. The technique is based on Newton iterations and we initialised it with the solution of linear triangulation. For the linear triangulation, we used the normalized implementation by the same authors~\cite{hartley2003multiple}. 
For the ($\ell_2$,$\ell_\infty$)-norm, we used the SOCP implementation provided by Kahl~\cite{kahl2008multiple} and, finally, for the ($\ell_\infty$,$\ell_\infty$)-norm minismisation, we used the linear program (LP) defined in~\eqref{eq:LP} with $\boldsymbol{c}=\boldsymbol{0}$. Of course, this just returns a consistent estimate and does not fully minimise the norm.

\subsection{Numerical simulation of Conjecture~\ref{conj_upbound}}
We simulated the setup explained in Conjecture~\ref{conj_upbound}. In order to generate cameras uniformly at random that see the whole region of interest, we use rejection sampling; i.e., we repeatedly generate a camera centre and rotation matrix uniformly at random and reject ones that do not see the whole region of interest. To generate rotation matrices uniformly at random we use the technique outlined in \cite{Arvo:1992}. 

The results are shown in Fig.~\ref{fig_verification_conj1}, for $\ell_\infty$-norm and $\ell_2$-norm bounded noise. As expected, consistent algorithms matching the norm of the noise have a quadratic decay; i.e., the LP has a quadratic decay for $\ell_\infty$-norm bounded noise and the minimum ($\ell_2$,$\ell_\infty$)-norm has a quadratic decay for $\ell_2$-norm bounded noise.

Non-consistent techniques, such as the ($\ell_2$,$\ell_2$)-norm minimisation, can perform well for small $M$, but, as $M$ increases, fail to reach the quadratic decay.

\subsection{Numerical simulation of Theorem~\ref{th_upbound_3d}}
Finally, we present a simulation to experimentally test Theorem~\ref{th_upbound_3d}. The algorithms are the same as for the simulation of Conjecture~\ref{conj_upbound} and the results are shown in Fig.~\ref{fig_verification_thm2}. As can be seen in the figure, all algorithms behave as expected.
\section{Conclusion}
\label{sec_conclusion}
We presented an analysis of the accuracy of multi-camera systems and their error decay rate with respect to the number of cameras. In doing so, we derived fundamental scaling laws stating that, under certain conditions, the accuracy of a multi-camera imaging system in reconstructing 3-D points increases quadratically with respect to the number of cameras. 

We also analysed the performance of state-of-the-art algorithms with respect to their error decay rate. To do this, we introduced the notion of consistency and showed that consistent reconstruction algorithms achieve the optimal quadratic error decay. Furthermore, we showed that ($\ell_\infty$,$\ell_\infty$)-norm based minimisation is consistent and, in addition, presented two simple linear programs that are also consistent.

\section*{Acknowledgement}
The authors would like to thank Richard Hartley, Fredrik Kahl and Pascal Frossard for their advice and fruitful discussions, which have greatly improved the manuscript.

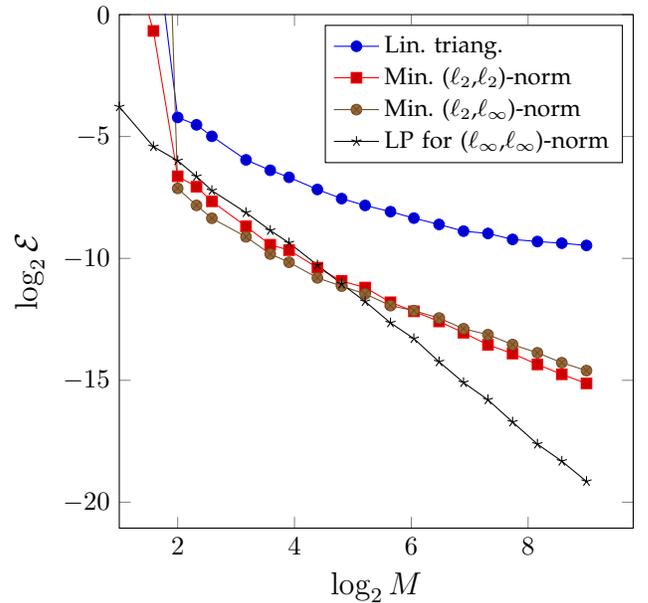
\begin{figure}[t!]
	\centering
	\begin{center}
\begin{tikzpicture}[scale=1]
	\begin{axis}[
		width=0.95\columnwidth,
        height=0.95\columnwidth,
		xlabel=\large{$\log_2 M$},
		ylabel={\large{$\log_2 \mathcal{E}$}},
		legend style={font=\small},xmin=1,ymax=0,
		legend cell align=left,
		]

		\addplot coordinates{
(1.00000000, 3.40256817)
(1.58496250, 3.86683856)
(2.00000000, -4.21874012)
(2.32192809, -4.52210287)
(2.58496250, -4.99246218)
(3.16992500, -5.95969289)
(3.58496250, -6.38606267)
(3.90689060, -6.67393136)
(4.39231742, -7.17239538)
(4.80735492, -7.54792305)
(5.20945337, -7.82784836)
(5.64385619, -8.08345356)
(6.04439412, -8.34869603)
(6.47573343, -8.60810066)
(6.89481776, -8.88317615)
(7.31288296, -8.97711701)
(7.73470962, -9.22085232)
(8.15987134, -9.30872814)
(8.57742883, -9.37702834)
(9.00000000, -9.46868942)
		};
		\addlegendentry{\small Lin. triang.}
		
		\addplot coordinates{
(1.00000000, 4.14483112)
(1.58496250, -0.67055659)
(2.00000000, -6.63090307)
(2.32192809, -7.06065164)
(2.58496250, -7.66051927)
(3.16992500, -8.68658573)
(3.58496250, -9.44023576)
(3.90689060, -9.66341200)
(4.39231742, -10.38000349)
(4.80735492, -10.92885032)
(5.20945337, -11.19613516)
(5.64385619, -11.80034709)
(6.04439412, -12.17117075)
(6.47573343, -12.58623274)
(6.89481776, -13.05025505)
(7.31288296, -13.54527823)
(7.73470962, -13.90473023)
(8.15987134, -14.35058553)
(8.57742883, -14.75557540)
(9.00000000, -15.12857285)
		};
		\addlegendentry{\small Min. ($\ell_2$,$\ell_2$)-norm}
		
		\addplot coordinates{
(1.00000000, 30.74739036)
(1.58496250, 24.50536271)
(2.00000000, -7.12835768)
(2.32192809, -7.82558521)
(2.58496250, -8.35394305)
(3.16992500, -9.11513048)
(3.58496250, -9.82048692)
(3.90689060, -10.15735553)
(4.39231742, -10.79887951)
(4.80735492, -11.13680694)
(5.20945337, -11.45653377)
(5.64385619, -11.93774379)
(6.04439412, -12.14407076)
(6.47573343, -12.44529969)
(6.89481776, -12.88533685)
(7.31288296, -13.12896048)
(7.73470962, -13.53471936)
(8.15987134, -13.87278818)
(8.57742883, -14.27933842)
(9.00000000, -14.60119634)
		};
		\addlegendentry{\small Min. ($\ell_2$,$\ell_\infty$)-norm}
	
		\addplot coordinates{
(1.00000000, -3.78312071)
(1.58496250, -5.41737540)
(2.00000000, -5.99233329)
(2.32192809, -6.64016672)
(2.58496250, -7.21895096)
(3.16992500, -8.12242159)
(3.58496250, -8.85286839)
(3.90689060, -9.36249405)
(4.39231742, -10.27329659)
(4.80735492, -11.04537671)
(5.20945337, -11.76883966)
(5.64385619, -12.64353293)
(6.04439412, -13.29033376)
(6.47573343, -14.24101884)
(6.89481776, -15.09353160)
(7.31288296, -15.79901022)
(7.73470962, -16.71056084)
(8.15987134, -17.61711297)
(8.57742883, -18.31766221)
(9.00000000, -19.14800516)
		};
		\addlegendentry{\small LP for ($\ell_\infty$,$\ell_\infty$)-norm}
		
	\end{axis}
\end{tikzpicture}
	\end{center}
	
	\caption{Verification of Theorem~\ref{th_upbound_3d}. Expected squared error ($\mathcal{E}$) as a function of the number of cameras ($M$).
	}
	\label{fig_verification_thm2}
\end{figure}

\bibliographystyle{ieeetr}   
\bibliography{crmg}   

\begin{IEEEbiography}[{\includegraphics[width=1in,height=1.25in,clip,keepaspectratio]{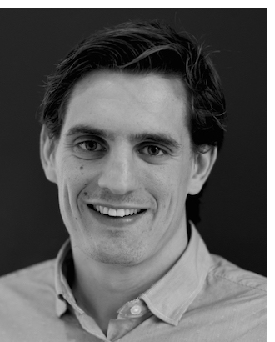}}]{Adam Scholefield}
	    is currently a post-doctoral researcher in the Audiovisual Communications Laboratory at the \'Ecole Polytechnique F\'ed\'erale de Lausanne (EPFL), Switzerland. He received the MEng. degree in Electrical and Electronic Engineering from Imperial College London, in 2007 and, after a brief interruption of studies to represent Team GB in the 2012 Olympics, the Ph.D. degree, from the same university, in 2013. His research interests include computational imaging, computer vision, localisation, optics and sampling theory.
\end{IEEEbiography}

\begin{IEEEbiography}[{\includegraphics[width=1in,height=1.25in,clip,keepaspectratio]{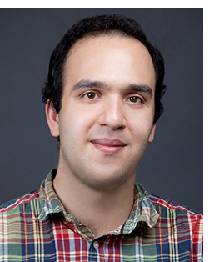}}]{Alireza Ghasemi}
	    is a data scientist at ELCA Informatique, Lausanne, Switzerland. Before, he joined EPFL in 2011 as a PhD candidate at the Audiovisual Communications Laboratory and finished his studies in 2016, under the supervision of Prof. Martin Vetterli and Dr. Adam James Scholefield. He received his BSc in Software Engineering and MSc in Artificial Intelligence, in 2009 and 2011 respectively, from the Sharif University of Technology, Tehran, Iran. His research interests include machine learning, 3-D computer vision, geometric reconstruction, natural language processing, and light-field image processing. 
\end{IEEEbiography}
    
\begin{IEEEbiography}[{\includegraphics[width=1in,height=1.25in,clip,keepaspectratio]{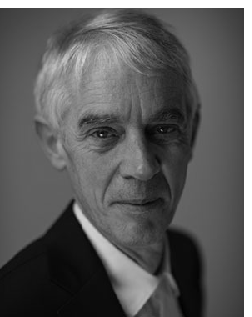}}]{Martin Vetterli}
     received the EE degree from ETHZ in 1981, the MS from Stanford University in 1982, and the Doctorate from EPFL in 1986. He was an Assistant/Associate  Professor in EE at Columbia University, and then an Associate/Full Professor in EECS at the UC Berkeley. In 1995, he joined the EPFL as a Full Professor. From 2004 to 2011 he was VP of EPFL for international affairs, and from 2011 to 2012, he was the Dean of IC. From 2013 to 2016 he was President of the Swiss NSF and since January 2017, he is President of EPFL. He works in the areas of electrical engineering, computer sciences and applied mathematics and this led to about 170 journals papers, as well as about 30 patents that led to technology transfer to high-tech companies and the creation of several start-ups. He is also the co-author of three textbooks.  His prizes include best paper awards from EURASIP in 1984 and of the IEEE Signal Processing Society in 1991, 1996 and 2006, the Swiss National Latsis Prize in 1996, the SPIE Presidential award in 1999, the IEEE SP Technical Achievement Award in 2001 and the IEEE SP Society Award in 2010 and the IEEE Jack S. Kilby Signal Processing Medal in 2017. He is a Fellow of IEEE, of ACM, ISI highly cited researcher and a foreign member of the NAE.
\end{IEEEbiography}

\end{document}